\documentclass[a4paper,10pt]{article}
\usepackage[utf8]{inputenc}
\usepackage{subfig}

\usepackage{amsthm,amsmath,amssymb}
\usepackage{todonotes}

\newcommand{\poly}{\mathrm{poly}}

\newcommand{\var}{\mathsf{var}}
\newcommand{\MM}{\mathrm{MM}}

\newtheorem{theorem}{Theorem}[section]
\newtheorem{lemma}[theorem]{Lemma}
\newtheorem{corollary}[theorem]{Corollary}
\newtheorem{proposition}[theorem]{Proposition}
\newtheorem{observation}[theorem]{Observation}
\newtheorem{clm}[theorem]{Claim}
\title{QBF as an Alternative to Courcelle's Theorem}
\author{Michael Lampis\thanks{Université Paris-Dauphine, PSL Research University, CNRS, UMR 7243} \and Stefan Mengel\thanks{CNRS, CRIL UMR 8188} \and Valia Mitsou\thanks{Université Paris-Diderot, IRIF, CNRS, UMR 8243}}

\begin{document}

\maketitle

\begin{abstract}
We propose reductions to quantified Boolean formulas (QBF) as a new approach to showing fixed-parameter linear algorithms for problems parameterized by treewidth. We demonstrate the feasibility of this approach by giving new algorithms for several well-known problems from artificial intelligence that are in general complete for the second level of the polynomial hierarchy. By reduction from QBF we show that all resulting algorithms are essentially optimal in their dependence on the treewidth. Most of the problems that we consider were already known to be fixed-parameter linear by using Courcelle's Theorem or dynamic programming, but we argue that our approach has clear advantages over these techniques: on the one hand, in contrast to Courcelle's Theorem, we get concrete and tight guarantees for the runtime dependence on the treewidth. On the other hand, we avoid tedious dynamic programming and, after showing some normalization results for CNF-formulas, our upper bounds often boil down to a few lines.
% We revisit several reasoning problems that we known to have linear time algorithms on graphs of bounded treewidth, in particular the following problems:
% \begin{itemize}
%  \item 
% \end{itemize}
% While in prior work the dependence of the runtime on the treewidth was open, we give tight bounds: for all problems above there are algorithms with runtime $2^{2^{O(k)}}n$ for instances of size $n$ and treewidth $k$. Moreover, assuming the exponential time hypothesis, a standard assumption in parameterized complexity, there is no algorithm with runtime $2^{2^{o(k)}}n^c$ for any constant $c$. We show all our results by relating the problems we consider to QBF-SAT.
\end{abstract}

\section{Introduction}

Courcelle's seminal theorem~\cite{Courcelle90} states that every graph property definable in monadic second order logic can be decided in linear time on graphs of constant treewidth. Here treewidth is the famous width measure used to measure intuitively how similar a graph is to a tree. While the statement of Courcelle's Theorem might sound abstract to the unsuspecting reader, the consequences are tremendous. Since a huge number of computational problems can be encoded in monadic second order logic, this gives automatic linear time algorithms for a wealth of problems in such diverse fields as combinatorial algorithms, artificial intelligence and databases; out of the plethora of such papers let us only cite~\cite{GottlobPW10,Dunne07} that treat problems that will reappear in this paper. This makes Courcelle's Theorem one of the cornerstones of the field of parameterized algorithms. 

Unfortunately, its strength comes with a price: while the runtime dependence on the size of the problem instance is linear, the dependence on the treewidth is unclear when using this approach. Moreover, despite recent progress (see e.g. the survey~\cite{LangerRRS14}) Courcelle's Theorem is largely considered impractical due to the gigantic constants involved in the construction. Since generally these constants are unavoidable~\cite{FrickG04}, showing linear time algorithms with Courcelle's Theorem can hardly be considered as a satisfying solution.

As a consequence, linear time algorithms conceived with the help of Courcelle's Theorem are sometimes followed up with more concrete algorithms with more explicit runtime guarantees often by dynamic programming or applications of a datalog approach~\cite{DvorakPW12,GottlobPW10,JaklPRW08}. Unfortunately, these hand-written algorithms tend to be very technical, in particular for decision problems outside of $\mathsf{NP}$. Furthermore, even this meticulous analysis usually gives algorithms with a dependance on treewidth that is a tower of exponentials.

The purpose of this paper is two-fold. On the one hand we propose reductions to QBF combined with the use of a known QBF-algorithm by Chen~\cite{Chen04} as a simple approach to constructing linear-time algorithms for problems beyond $\mathsf{NP}$ parameterized by treewidth. In particular, we use the proposed method in order to construct (alternative) algorithms for a variety of problems stemming from artificial intelligence: abduction, circumscription, abstract argumentation and the computation of minimal unsatisfiable sets in unsatisfiable formulas. The advantage of this approach over Courcelle's Theorem or tedious dynamic programming is that the algorithms we provide are almost straightforward to produce, while giving bounds on the treewidth that asymptotically match those of careful dynamic programming. On the other hand, we show that our algorithms are asymptotically best possible, giving matching complexity lower bounds. 

%Our problem pool will be the area of artificial intelligence that contains a variety of problems complete for classes in the second level of the polynomial hierarchy: abduction, circumscription, abstract argumentation and the computation of minimal unsatisfiable sets in unsatisfiable formulas. 

Our algorithmic approach might at first sight seem surprising: since QBF with a fixed number of alternations is complete for the different levels of the polynomial hierarchy, there are trivially reductions from all problems in that hierarchy to the corresponding QBF problem. So what is new about this approach? The crucial observation here is that in general reductions to QBF guaranteed by completeness do not maintain the treewidth of the problem. %So we might start with an instance of small treewidth and then reduce to a QBF-instance with high treewidth such that Chen's algorithm does not give us any meaningful runtime guarantees. 
Moreover, while Chen's algorithm runs in linear time, there is no reason for the reduction to QBF to run in linear time which would result in an algorithm with overall non-linear runtime.

The runtime bounds that we give are mostly of the form $2^{2^{O(k)}}n$ where $k$ is the treewidth and $n$ the size of the input. Furthermore, starting from recent lower bounds for QBF~\cite{LampisM17}, we also show that these runtime bounds are essentially tight as there are no algorithms with runtime $2^{2^{o(k)}}2^{o(n)}$ for the considered problems. Our lower bounds are based on the \emph{Exponential Time Hypothesis (ETH)} which posits that there is no algorithm for 3SAT with runtime $2^{o(n)}$ where $n$ is the number of variables in the input. ETH is by now widely accepted as a standard assumption in the fields of exact and parameterized algorithms for showing tight lower bounds, see e.g., the survey~\cite{LokshtanovMS11}. %and we show that we can use it to evidence that our upper bounds are essentially tight. 
We remark that our bounds confirm the observation already made in~\cite{MarxM16} that problems complete for the second level of the polynomial hierarchy parameterized by treewidth tend to have runtime double-exponential in the treewidth.

As a consequence, the main contribution of this paper is to show that reductions to QBF can be used as a simple technique to show algorithms with essentially optimal runtime for a wide range of problems.

\paragraph{Our Contributions.}
We show upper bounds of the form $2^{2^{O(k)}}n$ for instances of treewidth $k$ and size $n$ for abstract argumentation, abduction, circumscription and the computation of minimal unsatisfiable sets in unsatisfiable formulas. For the former three problems it was already known that there are linear time algorithms for bounded treewidth instances: for abstract argumentation, this was shown in~\cite{Dunne07} with Courcelle's theorem and a tighter upper bound of the form $2^{2^{O(k)}}n$ was given by dynamic programming in~\cite{DvorakPW12}. For abduction, there was a linear time algorithm in~\cite{GottlobPW10} for all abduction problems we consider and a $2^{2^{O(k)}}n$ algorithm based on a datalog encoding for some of the problems. The upper bound that we give for so-called \emph{necessity} is new. For circumscription, a linear time algorithm was known~\cite{GottlobPW10} but we are the first to give concrete runtime bounds. Finally, we are the first to give upper bounds for minimal unsatisfiable subsets for CNF-formulas of bounded treewidth.

We complement our upper bounds with ETH-based lower bounds for all problems mentioned above, all of which are the first such bounds for these problems.

Finally, we apply our approach to abduction with $\subseteq$-preferences but giving a linear time algorithm with triple exponential dependence on the treewidth, refining upper bounds based on Courcelle's theorem~\cite{GottlobPW10} by giving an explicit treewidth dependence.

\section{Preliminaries}

In this section, we only introduce notation that we will use in all parts of the paper. The background for the problems on which we demonstrate our approach will be given in the individual sections in which these problems are treated. 
\subsection{Treewidth}

Throughout this paper, all graphs will be undirected and simple unless explicitly stated otherwise. A \emph{tree decomposition} $(T, (B_t)_{t\in T})$ of a graph $G=(V,E)$ consists of a tree $T$ and a subset $B_t\subseteq V$ for every node $t$ of $T$ with the following properties:
\begin{itemize}
 \item every vertex $v\in V$ is contained in at least one set $B_t$,
 \item for every edge $uv\in E$, there is a set $B_t$ that contains both $u$ and $v$, and
 \item for every $v\in V$, the set $\{t\mid v\in B_t\}$ induces a subtree of $T$.
\end{itemize}
The last condition is often called the connectivity condition. The sets $B_t$ are called \emph{bags}. The \emph{width} of a tree decomposition is $\max_{t\in T}(|B_t|)-1$. The \emph{treewidth} of $G$ is the minimum width of a tree decomposition of $G$. We will sometimes tacitly use the fact that any tree decomposition can always be assumed to be of size linear in $|V|$ by standard simplifications.
Computing the treewidth of a graph is $\mathsf{NP}$-hard~\cite{ArnborgP89}, but for every fixed $k$ there is a linear time algorithm that decides if a given graph has treewidth at most $k$ and if so computes a tree decomposition witnessing this~\cite{Bodlaender96}.

A tree decomposition is called \emph{nice} if every node $t$ of $T$ is of one of the following types:
\begin{itemize}
 \item \textbf{leaf node:} $t$ is a leaf of $T$.
 \item \textbf{introduce node:} $t$ has a single child node $t'$ and $B_t=B_{t'}\cup \{v\}$ for a vertex $v\in V\setminus B_{t'}$. 
 \item \textbf{forget node:} $t$ has a single child node $t'$ and $B_t=B_{t'}\setminus \{v\}$ for a vertex $v\in B_{t'}$. 
 \item \textbf{join node:} $t$ has exactly two children $t_1$ and $t_2$ with $B_{t} = B_{t_1} = B_{t_2}$.
\end{itemize}
Nice tree decompositions were introduced in~\cite{Kloks94} where it was also shown that given a tree decomposition of a graph $G$, one can in linear time compute a nice tree decomposition of $G$ with the same width.

\subsection{CNF formulas}

A \emph{literal} is a propositional variable or the negation of a propositional variable. A \emph{clause} is a disjunction of literals and a CNF-formula is a conjunction of clauses. 
For technical reasons we assume that there is an injective mapping from the variables in a CNF formula $\phi$ to $\{0, \ldots, cn\}$ 
 for an arbitrary but fixed constant $c$ where $n$ is the number of variables in $\phi$ and that we can evaluate this mapping in constant time. This assumption allows us to easily create lists, in linear time in $n$, which store data assigned to the variables that we can then look up in constant time. Note that formulas in the DIMACS format~\cite{challenge1993satisfiability}, the standard encoding for CNF formulas, generally have this assumed property. Alternatively, we could use perfect hashing to assign the variables to integers, but this would make some of the algorithms randomized.

Let $\phi$ and $\phi'$ be two CNF formulas. We say that $\phi$ is a projection of $\phi'$ if and only if $\var(\phi) \subseteq \var(\phi')$ and $a: \var(\phi)\rightarrow \{0,1\}$ is a model of $\phi$ if and only if $a$ can be extended to a model of $\phi'$.

%\begin{figure}
%\centering
% \begin{subfigure}[b]{0.45\textwidth}
%	\centering
%        \includegraphics[width=0.4\textwidth]{primal}
%        \caption{Primal graph}
%        \label{fig:primal}
%    \end{subfigure}
%    ~ %add desired spacing between images, e. g. ~, \quad, \qquad, \hfill etc. 
%      %(or a blank line to force the subfigure onto a new line)
%    \begin{subfigure}[b]{0.45\textwidth}
%    \centering
%        \includegraphics[width=0.4\textwidth]{incidence}
%        \caption{Incidence graph}
%        \label{fig:incidence}
%    \end{subfigure}
%\caption{Primal and incidence graphs for $\phi=(\neg x \vee z) \wedge (x \vee y \vee \neg w) \wedge (\neg z \vee w)$.}
%\label{fig:graphs}
%\end{figure}

\begin{figure}[h]
  \begin{center}
    \subfloat[Primal graph]{
      \includegraphics[width=0.25\textwidth]{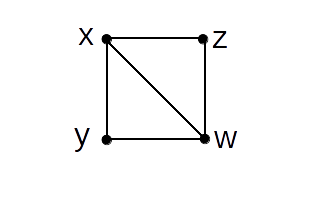}
      \label{fig:primal}
                         }
   \hspace*{0.3in}
    \subfloat[Incidence graph]{
      \includegraphics[width=0.25\textwidth]{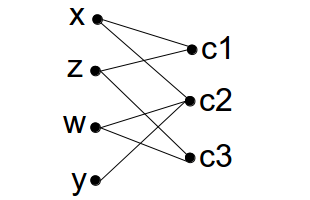}
      \label{fig:incidence}
                         }
\caption{Primal and incidence graphs for $\phi=(\neg x \vee z) \wedge (x \vee y \vee \neg w) \wedge (\neg z \vee w)$.}
\label{fig:graphs}
  \end{center}
\end{figure}

To every CNF formula $\phi$ we assign a graph called \emph{primal graph} whose vertex set is the set of variables of $\phi$. Two vertices are connected by an edge if and only if they appear together in a clause of $\phi$ (see Figure~\ref{fig:primal}). The \emph{primal treewidth} of a CNF formula is the treewidth of its primal graph.
We will also be concerned with the following generalization of primal treewidth: the \emph{incidence graph} of a CNF formula has as vertices the variables \emph{and} the clauses of the formula. Two vertices are connected by an edge if and only if one vertex is a variable and the other is a clause such that the variable appears in the clause (see Figure~\ref{fig:incidence}). The \emph{incidence treewidth} of a formula is then the treewidth of its incidence graph.

%TODO!!!!!!!!!!!!!!!!!!!!You had them reversed
It is well-know that the primal treewidth of a CNF-formula can be arbitrarily higher than the incidence treewidth (for example consider a single clause of size $n$). The other way round, formulas of primal treewidth $k$ can easily be seen to be of incidence treewidth at most $k+1$~\cite{FischerMR08}.

%TODO!!!!!!!!!!!!!! I changed section->subsection. I also modified the intro paragraph. I checked the Samer-Szeider paper, I couldn't find any reference as to this particular transformation, only the fact that tw dominates {arity, tw*} in page 9.
\subsection{From primal to incidence treewidth}\label{sec:CNF}

%In this section, we will show several linear time transformations of CNF-formulas that will be useful in the remainder of the paper. Some of these transformations are somewhat tedious but none of them are particularly hard. 

While in general primal and incidence treewidth are two different parameters, in this section we argue that when dealing with CNF formulas we don't need to distinguish between the two: first, since incidence treewidth is more general, the lower bounds for primal treewidth transfer automatically to it; second, while the same cannot generally be said for algorithmic results, it is easy to see that the primal treewidth is bounded by the product of the incidence treewidth the arity (clause size), so it suffices to show that we can transform any CNF formula to an equivalent one having bounded arity while roughly maintaining its incidence treewidth. Proposition \ref{prop:threeCNF} suggests a linear time transformation achieving this. In the following we can then interchangably work with incidence treewidth or primal treewidth, whichever is more convenient in the respective situation.

%\todo{"have a restricted in which the size of the clauses is bounded and the number of clauses in which a variable appears is also bounded." Do we need the second?} No!
%There are of course well-known standard reductions for this. Proposition \ref{prop:threeCNF} suggests that we can apply them in linear time while roughly maintaining the treewidth.

\begin{proposition}\label{prop:threeCNF}
There is an algorithm that, given a CNF formula $\phi$ of incidence treewidth $k$, computes in time $2^{O(k)} |\phi|$ a 3CNF formula $\phi'$ of incidence treewidth $O(k)$ with $\var(\phi) \subseteq \var(\phi')$ such that $\phi$ is a projection of $\phi'$.
\end{proposition}
\begin{proof}(Sketch)
We use the classic reduction from SAT to 3SAT that cuts big clauses into smaller clauses by introducing new variables. During this reduction we have to take care that the runtime is in fact linear and that we can bound the treewidth appropriately. For the complete proof see Appendix~\ref{app:threeCNF}.
\end{proof}

%We remark that by using Tseitin transformation in the construction above we could have constructed $\phi'$ in such a way that every model of $\phi$ can be extended to $\phi'$ in exactly one way which could be interesting for applications that require the number of models to remain the same.\todo{do we actually want to do this?}

It is well-known that if the clauses in a formula $\phi$ of incidence treewidth $k$ have at most size $d$, then the primal treewidth of $\phi$ is at most $(k+1)d$, see e.g.~\cite{FischerMR08} so the following result follows directly.

\begin{corollary}\label{cor:twtransfer}
 There is an algorithm that, given a CNF-formula $\phi$ of incidence treewidth $k$, computes in time $O(2^k |\phi|)$ a 3CNF-formula $\phi'$ of primal treewidth $O(k)$ such that $\phi$ is a projection of $\phi'$.
\end{corollary}

We will in several places in this paper consider Boolean combinations of functions expressed by CNF formulas of bounded treewidth. The following technical lemma states that we can under certain conditions construct CNF formulas of bounded treewidth for the these Boolean combinations.

\begin{lemma}\label{lem:combination}
\begin{itemize}
 \item[a)] There is an algorithm that, given a 3CNF-formula $\phi$ and a tree decomposition $(T,(B_t)_{t\in T})$ of its incidence graph of width $O(k)$, computes in time $\poly(k) n$ a CNF-formula $\phi'$ and a tree decomposition $(T',(B_t')_{t\in T})$ of the incidence graph of $\phi'$ such that $\neg\phi$ is a projection of $\phi'$, for all $t\in T$ we have $B_t'\cap \var(\phi) = B_t$ and the width of $(T',(B_t')_{t\in T})$ is $O(k)$.
 \item[b)] There is an algorithm that, given two 3CNF-formulas $\phi_1, \phi_2$ and two tree decompositions $(T,(B^i_t)_{t\in T})$ for $i = 1,2$ of the incidence graphs of $\phi_i$ of width $O(k)$ such that for every bag either $B_t^1\cap B_t^2 = \emptyset$ or $B_t^1\cap \var(\phi_1) = B_t^2\cap \var(\phi_2)$, computes in time $\poly(k) n$ a tree decomposition $(T',(B_t')_{t\in T})$ of the incidence graph of $\phi_1\land \phi_2$ such that $\phi' \equiv \phi_1\land \phi_2$, for all $t\in T$ we have $B_t^1\cup B_t^2 = B_t'$ and the width of $(T,(B_t')_{t\in T})$ is $O(k)$.
\end{itemize}
\end{lemma}
\begin{proof}
 a) Because every clause has at most $3$ literals, we assume w.l.o.g. that every bag $B$ that contains a clause $C$ contains also all variables of $C$.
 
 In a first step, we add for every clause $C= \ell_1\lor \ell_2\lor \ell_3$ a variable $x_C$ and substitute $C$ by clauses with at most $3$-variables encoding the constraint $\mathcal{C} = x_C \leftrightarrow l_1\lor l_2\lor l_3$ introducing some new variables. The result is a CNF-formula $\phi_1$ in which every assignment $a$ to $\var(\phi)$ can be extended uniquely to a satisfying assignment $a_1$ and in $a_1$ the variable $x_C$ is true if and only if $C$ is satisfied by $a$. Note that, since every clause has at most $3$ variables, the clauses for $\mathcal{C}$ can be constructed in constant time. Moreover, we can construct a tree decomposition of width $O(k)$ for $\phi_1$ from that of $\phi$ by adding all new clauses for $\mathcal{C}$ and $x_C$ to every bag containing $C$.
 
 In a next step, we introduce a variable $x_t$ for every $t\in T$ and a constraint $\mathcal{T}$ defining $x_t \leftrightarrow (x_{t_1} \land x_{t_2} \land \bigwedge_{C\in B_t} x_C)$ where $t_1, t_2$ are the children of $t$ and the variables are omitted in case they do not appear. The resulting CNF formula $\phi_2$ is such that every assignment $a$ to $\var(\phi)$ can be uniquely extended to a satisfying assignment $a_2$ of $\phi_2$ and $x_t$ is true in $a_2$ if and only if all clauses that appear in the subtree of $T$ rooted in $t$ are satisfied by $a$. Since every constraint $\mathcal{T}$ has at most $k$ variables, we can construct the 3CNF-formula simulating it in time $O(k)$, e.g.~by Tseitin transformation. We again bound the treewidth as before.
 
 The only thing that remains is to add a clause $\neg x_r$ where $r$ is the root of $T$. This completes the proof of a).
 
 b) We simply set $B_t' = B_t^1 \cup B_t^2$. It is readily checked that this satisfies all conditions.
\end{proof}

\begin{lemma}\label{lem:subset}
 There is an algorithm that, given a 3CNF formula $\phi$ with a tree decomposition $(T, (B_t)_{t\in T})$ of width $k$ of the incidence graph of $\phi$ and sequences of variables $X:=(x_1, \ldots , x_\ell)$, $Y=(y_1, \ldots, y_\ell)\subseteq \var(\phi)^\ell$ such that for every $i \in [\ell]$ there is a bag $B_t$ with $\{x_i, y_i\} \in B_t$, computes in time $\poly(k)|\phi|$ a formula $\psi$ that is a projection of $X\subseteq Y = \bigwedge_{i=1}^\ell (x_i \le y_i)$ and a tree decomposition $(T, (B_t)_{t\in T})$ of $\psi$ of width $O(1)$. The same is true for $\subset$ instead of $\subseteq$.
\end{lemma}
\begin{proof}
 For the case $\subseteq$, $\psi$ is simply $\bigwedge_{i=1}^\ell (x_i \le y_i)= \bigwedge_{i=1}^\ell \neg x_i \lor y_i$. $\psi$ satisfies all properties even without projection and with the same tree decomposition.
 
 The case $\subset$ is slightly more complex. We first construct $\bigwedge_{i=1}^\ell (x_i = y_i) = \bigwedge_{i=1}^\ell (\neg x_i \lor y_i) \land (x_i \lor \neg y_i)$. Then we apply Lemma~\ref{lem:combination} a) to get a CNF formula that has $X\ne Y$ as a projection. Finally, we use Lemma~\ref{lem:combination} to get a formula for $X\subset Y = (X\subseteq Y) \land (X\ne Y)$. It is easy to check that this formula has the right properties for the tree decomposition.
\end{proof}

\section{$2$-QBF} 

Our main tool in this paper will be QBF, the quantified version of CNF. 
In particular, we will be concerned with the version of QBF which only has two quantifier blocks which is often called $2$-QBF. Let us recall some standard definitions. A $\forall \exists$-QBF 
%\todo{Isn't it just $\forall\exists-CNF$ or $\forall\exists \phi$ ?} 
is a formula of the form $\forall X \exists Y \phi(X, Y)$ where $X$ and $Y$ are disjoint vectors of variables and $\phi(X, Y)$ is a CNF-formula called the \emph{matrix}. We assume the usual semantics for $\forall \exists$-QBF. Moreover, we sometimes consider Boolean combinations of QBF-formulas which we assume to be turned into prenex form again with the help of the usual transformations.

It is well-known that deciding if a given $\forall \exists$-QBF is true is complete for the second level of the polynomial hierarchy,  and thus generally considered intractable. Treewidth has been used as an approach for finding tractable fragments of $\forall \exists$-QBF and more generally bounded alternation QBF. Let us define the primal (resp.\ insidence) treewidth of a $\forall \exists$-QBF to be the primal (resp.\ incidence) treewidth of the underlying CNF formula. 
%\todo{is that really needed?}
Chen~\cite{Chen04} showed the following result.
\begin{theorem}\cite{Chen04}\label{thm:hubie}
There is an algorithm that given a $\forall\exists$-QBF of primal treewidth $k$ decides in time $2^{2^{O(k)}}|\phi|$ if $\phi$ is true.
\end{theorem}
We note that the result of~\cite{Chen04} is in fact more general than what we state here. In particular, the paper gives a more general algorithm for $i$-QBF with running time $2^{2^{{\cdot^{\cdot^{\cdot}}}^{O(k)}}} |\phi|$, where the height of the tower of exponentials is $i$. For readability we will restrict ourselves to the case $i=2$ that is general enough for our needs. We also remark that Chen does not state that his algorithm in fact works in linear time. We sketch in Appendix~\ref{app:hubie} why there is indeed a linear runtime bound.

In the later parts of this paper, we require a version of Theorem~\ref{thm:hubie} for incidence treewidth which fortunately follows directly from Theorem~\ref{thm:hubie} and Corollary~\ref{cor:twtransfer}.
\begin{corollary}\label{cor:hubie}
There is an algorithm that given a $\forall\exists$-QBF of incidence treewidth $k$ decides in time $2^{2^{O(k)}} |\phi|$ if $\phi$ is true.
\end{corollary}

We remark that general QBF of bounded treewidth without any restriction on the quantifier prefix is $\mathsf{PSPACE}$-complete~\cite{AtseriasO14}, and finding tractable fragments by taking into account the structure of the prefix and notions similar to treewidth is quite an active area of research, see e.g.~\cite{abs-1711-02120,EibenGO16}.

To show tightness of our upper bounds, we use the following theorem from~\cite{LampisM17}.

\begin{theorem}\label{thm:lampisM}
 There is no algorithm that, given a $\forall\exists$-QBF $\phi$ with $n$ variables and primal treewidth $k$, decides if $\phi$ is true in time $2^{2^{o(k)}} 2^{o(n)}$, unless ETH is false.
\end{theorem}

\section{Abstract Argumentation}\label{sec:argumentation}

Abstract argumentation is an area of artificial intelligence which tries to assess the acceptability of arguments within a set of possible arguments based only the relation between them, i.e., which arguments defeat which. Since its creation in~\cite{Dung95}, abstract argumentation has developed into a major and very active subfield. In this section, we consider the most studied setting introduced in~\cite{Dung95}.

An argumentation framework is a pair $F=(A, R)$ where $A$ is a finite set and $R\subseteq A\times A$. The elements of $A$ are called \emph{arguments}. The elements of $R$ are called the \emph{attacks} between the arguments and we say for $a,b\in A$ that $a$ attacks $b$ if and only if $ab\in R$. 
A set $S\subseteq A$ is called \emph{conflict-free} if and only if there are no $a,b\in S$ such that $ab\in R$. We say that a vertex $a$ is \emph{defended} by $S$ if for every $b$ that attacks $a$, i.e. $ba\in R$, there is an argument $c\in S$ that attacks $b$. The set $S$ is called \emph{admissible} if and only if it is conflict-free and all elements of $S$ are defended by $S$. An admissible set $S$ is called \emph{preferred} if and only if it is subset-maximal in the set of all admissible sets.

There are two main notions of acceptance: A set $S$ of arguments is accepted \emph{credulously} if and only if there is a preferred admissible set such that $S\subseteq S'$. The set $S$ is accepted \emph{skeptically} if and only if for all preferred admissible sets $S'$ we have $S\subseteq S'$. Both notions of acceptance have been studied extensively in particular with the following complexity results: it is $\mathsf{NP}$ hard to decide, given an argumentation framework $F=(A,R)$ and a set $S\subseteq A$, if $S$ is credulously accepted. For skeptical acceptance, the analogous decision problem is $\Pi_2^p$-complete~\cite{DunneB02}. Credulous acceptance is easier to decide, because when $S$ is contained in any admissible set $S'$ then it is also contained in a preferred admissible set $S''$: a simple greedy algorithm that adds arguments to $S'$ that are not in any conflicts constructs such an $S''$.

Concerning treewidth, after some results using Courcelle's Theorem~\cite{Dunne07}, it was shown in~\cite{DvorakPW12} by dynamic programming that credulous acceptance can be decided in time $2^{O(k)} n$ while skeptical acceptance can be decided in time $2^{2^{O(k)}}n$ for argument frameworks of size n and treewidth $k$.
Here an argument framework is seen as a directed graph and the treewidth is that of the underlying undirected graph. We reprove these results in our setting. To this end, we first encode conflict free sets in CNF. Given an argumentation framework $F=(A,R)$, construct a CNF formula $\phi_{cf}$ that has an indicator variable $x_a$ for every $a\in A$ as 
\[\phi_{cf} := \bigwedge_{ab\in R} \neg x_a \lor \neg x_b.\]
It is easy to see that the satisfying assignments of $\phi_{cf}$ encode the conflict-free sets for $F$. To encode the admissible sets, we add an additional variable $P_a$ for every $a\in A$ and define:
\[\phi_{d} := \phi_{cf} \land \bigwedge_{a\in A} ((\neg P_a \lor \bigvee_{b:ba\in R} x_b) \land \bigwedge_{b:ba\in R} (P_a\lor \neg x_b))\]
The clauses for each $P_a$ are equivalent to $P_a \leftrightarrow \bigvee_{b:ba \in R} x_b$, i.e., $P_a$ is true in a model if and only if $a$ is attacked by the encoded set. Thus by setting 
\[\phi_{adm} := \phi_d \land \bigwedge_{ba \in R} (\neg P_b \lor \neg x_a)\]
we get a CNF formula whose models restricted to the $x_a$ variables are exactly the admissible sets. We remark that in~\cite{LagniezLM15} the authors give a similar SAT-encoding for argumentation problems with slightly different semantics.
\begin{clm}\label{clm:argumentation}
 If $F$ has treewidth $k$, then $\phi_{adm}$ has incidence treewidth $O(k)$.
\end{clm}
\begin{proof}
 We start from a tree decomposition~$(T,(B_t)_{t\in T})$ of width $k$ of $F$ and construct a tree decomposition of $\phi_{adm}$. First note that $(T,(B_t)_{t\in T})$ is also a tree decomposition of the primal graph of $\phi_{cf}$ up to renaming each $a$ to $x_a$. For every $ba\in R$ there is thus a bag $B$ that contains both $b$ and $a$. We connect a new leaf to $B$ containing $\{C_{a,b},a,b\}$ where $C_{a,b}$ is a clause node for the clause $\neg x_a\lor \neg x_b$ to construct a tree decomposition of the primal graph of $\phi_{d}$.
 
 Now we add $P_a$ to all bags containing $x_a$, so that for every clause $P_a\lor \neg x_b$ we have a bag containing both variables, and we add new leaves for the corresponding clause nodes as before. Then we add for every clause $C_a := \neg P_a \lor \bigvee_{b:ba\in R} x_b$ the node $C_a$ to every bag containing $a$. This covers all edges incident to $C_a$ in the incidence graph of $\phi_d$ and since for every $a$ we only have one such edge, this only increases the width of the decomposition by a constant factor. We obtain a tree decomposition of width $O(k)$ for the incidence graph of $\phi_d$.
 
 The additional edges for $\phi_{adm}$ are treated similarly to above and we get a tree decomposition of width $O(k)$ of $\phi_{adm}$ of $\phi$ as desired.
\end{proof}

Combining Claim~\ref{clm:argumentation} with the fact that satisfiability of CNF-formulas of incidence treewidth $k$ can be solved in time $2^{O(k)}$, see e.g.~\cite{SamerS10}, we directly get the first result of~\cite{DvorakPW12}.

\begin{theorem}
 There is an algorithm that, given an argumentation framework $F=(A,R)$ of treewidth $k$ and a set $S\subseteq A$, decides in time $2^{O(k)}|A|$ if $S$ is credulously accepted.
\end{theorem}

We also give a short reproof of the second result of~\cite{DvorakPW12}.
\begin{theorem}\label{thm:argumentation}
  There is an algorithm that, given an argumentation framework $F=(A,R)$ of treewidth $k$ and a set $S\subseteq A$, decides in time $2^{2^{O(k)}}|A|$ if $S$ is skeptically accepted.
\end{theorem}
\begin{proof}
 Note that the preferred admissible sets of $F=(A,R)$ are exactly the subset maximal assignments to the $x_a$ that can be extended to a satisfying assignment of $\phi_{adm}$. Let $X:= \{x_a\mid a \in A\}$, then we can express the fact that an assignment is a preferred admissible set by
	\[ \phi'(X) = \exists P \forall X' \forall P' \left(\phi_{adm}(X, P) \land \left(\neg \phi_{adm}(X',P')\lor \neg(X\subset X')\right)\right)\] 
 where the sets $P$,$X'$ and $P'$ are defined analogously to $X$. Then $S$ does not appear in all preferred admissible sets if and only if
 \[ \exists X (\phi'(X) \land \bigvee_{a\in S}\neg x_a).\]
After negation we get
 \[ \forall X\forall P \exists X' \exists P' \left(\neg\phi_{adm}(X, P) \lor \left(\phi_{adm}(X',P')\land (X\subset X')\right)\lor \bigwedge_{a\in S} x_a\right)\] 
and using Lemma~\ref{lem:combination} afterwards yields a $\forall\exists$-QBF of incidence treewidth $O(k)$ that is true if and only if $S$ appears in all preferred admissible sets.
 This gives the result with Corollary~\ref{cor:hubie}.
\end{proof}

We remark that QBF encoding of problems in abstract argumentation have been studied in~\cite{EglyW06,ArieliC12}.

We now show that Theorem~\ref{thm:argumentation} is essentially tight.

\begin{theorem}
  There is no algorithm that, given an argumentation framework $F=(A,R)$ of size $n$ and treewidth $k$ and a set $S\subseteq A$, decides if $S$ is in every preferred admissible set of $F$ in time $2^{2^{o(k)}} 2^{o(n)}$, unless ETH is false.
\end{theorem}
\begin{proof}
  We use a construction from~\cite{DunneB02,DvorakPW12}: for a given $\forall\exists$-QBF $\forall Y \exists Z\phi$ in variables $Y\cup Z = \{x_1, \ldots, x_n\}$ and clauses $C_1, \ldots, C_m$, define $F_\phi = (A,R)$ with 
\begin{align*}
 A=& \{\phi, C_1, \ldots, C_m\} \lor \{x_i, \bar x_i\mid 1\le i\le n\}\cup \{b_1, b_2, b_3\}\\
  R=& \{(C_j, \phi)\mid 1 \le j \le m\}\cup \{ (x_i, \bar x_i), (\bar x_i, x_i)\mid 1 \le i\le n\}\\
  & \cup \{ (x_i, C_j)
   \mid x_i \text{ in $C_j$ }, 
  1 \le j \le m\} \cup \{ (\bar x_i, C_j)\mid \neg x_i \text{ in $C_j$ }, 1 \le j \le m\} \\
  &\cup\{(\phi, b_1), (\phi, b_2), (\phi, b_3), (b_1, b_2), (b_2, b_3), (b_3, b_1)\} \cup \{(b_1, z), (b_1, \bar z)\mid z\in Z\}
\end{align*}

One can show that $\phi$ is in every preferred admissible set of $F_{\phi}$ if and only if $\phi$ is true. Moreover, from a tree decomposition of the primal graph of $\phi$ we get a tree decomposition of $F$ as follows: we add every $\bar x_i$ to every bag that contains $x_i$ and we add $b_1, b_2, b_3$ to all bags. This increases the treewidth from $k$ to $2k+3$ and thus we get the claim with Theorem~\ref{thm:lampisM}.
\end{proof}

\section{Abduction}\label{thm:abduction}

In this section, we consider \emph{(propositional) abduction}, a form of non-monotone reasoning that aims to find explanations for observations that are consistent with an underlying theory. A \emph{propositional abduction problem} (short PAP) consists of a tuple $P=(V,H,M,T)$ where $T$ is a propositional formula called the \emph{theory} in variables $V$, the set $M\subseteq V$ is called the set of \emph{manifestations} and $H\subseteq V$ the set of hypotheses. We assume that $T$ is always in CNF. In abduction, one identifies a set $S\subseteq V$ with the formula $\bigwedge_{x\in S}x$. Similarly, given a set $S\subseteq H$, we define $T\cup S:= T\land \bigwedge_{x\in S} x$. A set $S\subseteq H$ is a \emph{solution} of the PAP, if $T\cup S\models M$, i.e., all models of $T\cup S$ are models of $M$.

There are three main problems on PAPs that have been previously studied:
\begin{itemize}
 \item Solvability: Given a PAP $P$, does it have a solution?
 \item Relevance: Given a PAP $P$ and $h\in H$, is $h$ contained in \emph{at least one} solution?
 \item Necessity: Given a PAP $P$ and $h\in H$, is $h$ contained in \emph{all} solutions?
\end{itemize}
The first two problems are $\Sigma_2^p$-complete while necessity is $\Pi_2^p$-complete~\cite{EiterG95}. In~\cite{GottlobPW10}, it is shown with Courcelle's Theorem that if the theory $T$ of an instance $P$ is of bounded treewidth, then all three above problems can be solved in linear time. Moreover,~\cite{GottlobPW10} gives an algorithm based on a Datalog-encoding that solves the solvability and relevance problems in time~$2^{2^{O(k)}}|T|$ on instances of treewidth $k$. Our first result gives a simple reproof of the latter results and gives a similar runtime for necessity.

\begin{theorem}
 There is a linear time algorithm that, given a PAP $P=(V,H,M,T)$ such that the incidence treewidth of $T$ is $k$ and $h\in H$, decides in time $2^{2^{O(k)}}|T|$ the solvability, relevance and necessity problems.
\end{theorem}
\begin{proof}
 We first consider solvability. We identify the subsets $S\subseteq H$ with assignments to $H$ in the obvious way. Then, for a given choice $S$, we have that $T\cup S$ is consistent if and only if
	\[\psi_1(S) := \exists X T(X) \land \bigwedge_{s_i\in H} (s_i \rightarrow x_i),\]
 is true where $X$ has a variable $x_i$ for every variable $v_i\in V$.
 Moreover, $T\cup S\models M$ if and only if 
 \[\psi_2:= \forall X' \left( \bigwedge_{s_i\in H} (s_i\rightarrow x'_i) \rightarrow \left( T(X') \rightarrow \bigwedge_{v_i\in M} x_i'\right) \right),\]
 where $X'$ similarly to $X$ has a variable $x_i$ for every variable $v_i\in V$.
 To get a $\forall \exists$-formula, we observe that the PAP has no solution if and only if
 \[\forall S \neg (\psi_1(S) \land \psi_2(S)) = \forall S \forall X \exists X' \neg(T(X) \land S\subseteq X|_H) \lor (S\subseteq X'|_H \land T(X') \land \neg \bigwedge_{v_i\in M}x_i')\]
 is true, where $X|_H$ denotes the restriction of $X$ to the variables of $H$.
 Now applying Lemmata~\ref{lem:combination} and \ref{lem:subset} in combination with de Morgan laws to express $\lor$ yields a $\forall\exists$-QBF of incidence treewidth $O(k)$ and the result follows with Corollary~\ref{cor:hubie}.
 
 For relevance, we simply add the hypothesis $h$ to $T$ and test for solvability.
 
 For necessity, observe that $h$ is in all solutions if and only if
 \[\forall S (\psi_1(S) \land \psi_2(S)) \rightarrow h,\]
 which can easily be brought into $\forall \exists$-QBF slightly extending the construction for the solvability case.
\end{proof}

Using the $\Sigma_2^p$-hardness reduction from~\cite{EiterG95}, it is not hard to show that the above runtime bounds are tight.

\begin{theorem}
  There is no algorithm that, given a PAP $P$ whose theory has primal treewidth $k$, decides decides solvability of $P$ in time $2^{2^{o(k)}} 2^{o(n)}$, unless ETH is false. The same is true for relevance and necessity of a variable $h$.
\end{theorem}
\begin{proof}
 Let $\phi' = \forall X \exists Y \phi$ be a $\forall \exists $-QBF with $X=\{x_1, \ldots, x_m\}$ and $Y=\{y_1, \ldots, y_\ell\}$. Define a PAP $P=(V,H,M,T)$ as follows
 \begin{align*}
  V &= X\cup Y\cup X' \cup \{s\}\\
  H &= X \cup X'\\
  M &= Y\cup \{s\}\\
  T &= \bigwedge_{i=1}^m (x_i \leftrightarrow \neg x_i') \land \underbrace{(\phi \rightarrow s\land \bigwedge_{j=1}^\ell y_j )}_\psi \land \bigwedge_{j=1}^\ell (s \rightarrow y_j)
 \end{align*}
where $X' = \{x_1', \ldots, x_m'\}$ and $s$ are fresh variables. It is shown in~\cite{EiterG95} that $\phi'$ is true if and only if $P$ has a solution. We show that $T$ can be rewritten into CNF-formula $T'$ with the help of Lemma~\ref{lem:combination}. The only non-obvious part is the rewriting of $\psi$. We solve this part by first negating into $(\phi \land ( \neg s\lor \bigvee_{j=1}^\ell \neg y_j )$ and observing that the second conjunct is just a clause, adding it to $\phi$ only increases the treewidth by $2$. Finally, we negate the resulting formula to get a CNF-formula for $\psi$ with the desired properties. The rest of the construction of $T'$ is straightforward. The claim then follows with Theorem~\ref{thm:lampisM}.

The result is a PAP with theory $T'$ of treewidth $O(k)$ and $O(n)$ variables and the result for solvability follows with Theorem~\ref{thm:lampisM}. As to the result for relevance and necessity, we point the reader to the proof of Theorem 4.3 in~\cite{EiterG95}. There for a PAP $P$ a new PAP $P'$ with three additional variables and $5$ additional clauses is constructed such that solvability of $P$ reduces to the necessity (resp. relevance) of a variable in $P'$. Since adding a fixed number of variables and clauses only increases the primal treewidth at most by a constant, the claim follows.
\end{proof}

\subsection{Adding $\subseteq$-Preferences}

In abduction there are often preferences for the solution that one wants to consider for a given PAP. One particular interesting case is $\subseteq$-preference where one tries to find (subset-)minimal solutions, i.e.~solutions $S$ such that no strict subset $S'\subseteq S$ is a solution. This is a very natural concept as it corresponds to finding minimal explanations for the observed manifestations. We consider two variations of the problems considered above, $\subseteq$-relevance and $\subseteq$-necessity. Surprisingly, complexity-wise, both remain in the second level of the polynomial hierarchy \cite{EiterG93}. Below we give a linear-time algorithm for these problems.

%\begin{itemize}
% \item $\subseteq$-relevance:
% \item $\subseteq$-necessity:
%\end{itemize}

%The definitions of $\subseteq$-relevance and $\subseteq$-necessity directly hint at an encoding for a linear time algorithm.

\begin{theorem}\label{thm:prefabduct1}
There is a linear time algorithm that, given a PAP $P=(V,H,M,T)$ such that the incidence treewidth of $T$ is $k$ and $h\in H$, decides in time $2^{2^{2^{O(k)}}}|T|$ the $\subseteq$-relevance and $\subseteq$-necessity problems.
\end{theorem}
\begin{proof}(sketch)
 We have seen how to express the property of a set $S$ being a solution as a formula $\psi(S)$ in the proof of Theorem~\ref{thm:abduction}. Then expressing that $S$ is a minimal model can written by 
 \[\psi'(S) := \psi(S) \land (\forall S' (S'\subseteq S \rightarrow \neg \psi(S'))).\]
 This directly yields QBFs for encoding the $\subseteq$-necessity and $\subseteq$-relevance problems as before which can again be turned into treewidth $O(k)$. The only difference is that we now have three quantifier alternations leading to a triple-exponential dependence on $k$ when applying the algorithm from~\cite{Chen04} %(see Appendix~\ref{app:hubie}).
\end{proof}
We remark that~\cite{GottlobPW10} already gives a linear time algorithm for $\subseteq$-relevance and $\subseteq$-necessity based on Courcelle's algorithm and thus without any guarantees for the dependence on the runtime. Note that somewhat disturbingly the dependence on the treewidth in Theorem~\ref{thm:prefabduct1} is triple-exponential. We remark that the lower bounds we could get with the techniques from the other sections are only double-exponential. Certainly, having a double-exponential dependency as in our other upper bounds would be preferable and thus we leave this as an open question.
% and in the remainder of this section we will show that we can in fact achieve this if we are willing to relax the requirement of having a linear time algorithm.

% As a first step, let us make a fundamental observation on solutions of PAPs already made in~\cite[Proposition 2.1]{EiterG95}.
% 
% \begin{observation}\label{obs:minabduct}
%  Let $P=(V, H, M, T)$ be a PAP with two solutions $S, S'$ such that $S'\subset S$. Then every set $S''$ with $S'\subseteq S'' \subset S$ is also a solution.
% \end{observation}
% \begin{proof}
%  $T\cup S''$ is satisfiable, because every model of $T\cup S$ is a model of $T\cup S''$ and such a model exists because $S$ is a solution. Moreover, all models of $T\cup S'$ are models of $T\cup S''$ and since all of the former satisfy $M$, we have $T\cup S''\models M$.
% \end{proof}
% 
% \begin{theorem}
% There is a linear time algorithm that, given a PAP $P=(V,H,M,T)$ such that the incidence treewidth of $T$ is $k$ and $h\in H$, decides in time $2^{2^{O(k)}}|T|\cdot |H|$ the $\subseteq$-relevance and $\subseteq$-necessity problems.
% \end{theorem}
% \begin{proof}[sketch]
%  We follow the same approach as in the proof of Theorem~\ref{obs:minabduct1}. The difference is that in $\psi(S) \land (\forall S' (S'\subseteq S \rightarrow \neg \psi(S')))$ we need due to Observation~\ref{obs:minabduct} in fact not check all $S'$ but only those subsets that are of size exactly one smaller than $S$. We will write a formula that checks exactly this. To this end, note that 
%  \begin{align*}
%  \psi' &= \forall S' 
%  \end{align*}
% 
% \end{proof}

\section{Circumscription}

In this section, we consider the problem of circumscription. To this end, consider a CNF-formula $T$ encoding a propositional function called the \emph{theory}. Let the variable set $X$ of $T$ be partitioned into three variable sets $P, Q, Z$. Then a model $a$ of $T$ is called \emph{$(P,Q,Z)$-minimal} if and only if there is no model $a'$ such that $a'|_P\subset a|_P$ and $a'|_Q = a|_Q$. In words, $a$ is minimal on $P$ for the models that coincide with it on $Q$. Note that $a$ and $a'$ can take arbitrary values on $Z$. We denote the $(P,Q,Z)$-minimal models of $T$ by $\MM(T, P, Q, Z)$. Given a CNF-formula $F$, we say that $\MM(T,P,Q,Z)$ entails $F$, in symbols $\MM(T,P,Q,Z)\models F$, if all assignments in $\MM(T,P,Q,Z)$ are models of $F$. The problem of \emph{circumscription} is, given $T, P,Q,Z$ and~$F$ as before, to decide if $\MM(T,P,Q,Z)\models F$.

Circumscription has been studied extensively and is used in many fields, see e.g.~\cite{Lifschitz:1994:CIR:186124.186130,McCarthy86}. We remark that circumscription can also be seen as a form of closed world reasoning which is equivalent to reasoning under the so-called extended closed world assumption, see e.g.~\cite{CadoliL94} for more context. On general instances circumscription is $\Pi_2^p$-complete~\cite{EiterG93} and for bounded treewidth instances, i.e.~if the treewidth of $T\land F$ is bounded, there is a linear time algorithm shown by Courcelle's Theorem~\cite{GottlobPW10}. There is also a linear time algorithm for the corresponding counting problem based on datalog~\cite{JaklPRW08}. We here give a version of the result from~\cite{GottlobPW10} more concrete runtime bounds.

\begin{theorem}\label{thm:circumUpper}
 There is an algorithm that, given an instance $T, P, Q, Z$ and $F$ of incidence treewidth $k$, decides if $\MM(T,P,Q,Z)\models F$ in time $2^{2^{O(k)}}(|T|+|F|)$.
\end{theorem}
\begin{proof}
 Note that we have $\MM(T,P,Q,Z)\models F$ if and only if for every assignment $(a_P, a_Q, a_Z)$ to $P,Q,Z$, we have that $(a_P, a_Q, a_Z)$ is not a model of $T$, or $(a_P, a_Q, a_Z)$ is a model of $F$ or there is a model $(a'_P, a'_Q, a'_Z)$ of $T$ such that $a'_P \subset a_P$ and $a'_Q= a_Q$. This can be written as a $\forall\exists$-formula as follows:
 \[\psi:=\forall P \forall Q\forall Z \exists P' \exists Z' ( \neg T(P,Q,Z) \lor F(P,Q,Z) \lor (T(P', Q, Z') \land P' \subset P)).
 \]
We first compute a tree decomposition of $T\land F$ of width $O(k)$ in time $2^{O(k)}(|T|+|F|)$. We can use Lemma~\ref{lem:combination}, Lemma~\ref{lem:subset} and Proposition~\ref{prop:threeCNF} to compute in time $\poly(k)(|T|+|F|)$ a CNF-formula $\phi$ such that the matrix of $\psi$ is a projection of $\phi$ and $\phi$ has incidence treewidth $O(k)$. Applying Corollary~\ref{cor:hubie}, yields the result.
\end{proof}

We now show that Theorem~\ref{thm:circumUpper} is essentially optimal by analyzing the proof in~\cite{EiterG93}.

\begin{theorem}
   There is no algorithm that, given an instance $T,P,Q, Z$ and $F$ of size $n$ and treewidth $k$, decides if $\MM(T,P,Q,Z)\models F$ in time $2^{2^{o(k)}} 2^{o(n)}$, unless ETH is false.
\end{theorem}
\begin{proof}
Let $\psi = \forall X \exists Y \phi$ be a QBF with $X= \{x_1, \ldots, x_m\}$ and $Y= \{y_1, \ldots, y_\ell\}$. We define the theory $T$ as follows:
\[T= \left(\bigwedge_{i=1}^m (x_i \ne z_i)\right) \land \left((u\land y_1\land \ldots y_\ell) \lor \phi\right),\]
where $z_1, \ldots , z_m$ and $u$ are fresh variables. Set $P=\var(T)$ and $Q=\emptyset$ and $Z$ the rest of the variables. In~\cite{EiterG93}, it is shown that $\MM(T,P,Q,Z)\models \neg u$ if and only if $\phi$ is true. Now using Lemma~\ref{lem:combination} we turn $\psi$ into a 2-QBF $\psi'$ with the same properties. Note that $\psi'$ has treewidth~$O(k)$ and $O(m+\ell)$ variables and thus the claim follows directly with Theorem~\ref{thm:lampisM}.\end{proof}

\section{Minimal Unsatisfiable Subsets}

Faced with unsatisfiable CNF-formula, it is in many practical settings highly interesting to find the sources of unsatisfiability. One standard way of describing them is by so-called minimal unsatisfiable subsets. A \emph{minimal unsatisfiable set (short MUS)} is an unsatisfiable set $\mathcal C$ of clauses of a CNF-formula such that every proper subset of $\mathcal C$ is satisfiable. The computation of MUS has attracted a lot of attention, see e.g.~\cite{LagniezB13,IgnatievPLM15,SilvaL11} and the references therein. 

In this section, we study the following question: given a CNF-formula $\phi$ and a clause $C$, is $C$ contained in a MUS of $\phi$? Clauses for which this is the case can in a certain sense be considered as not problematic for the satisfiability of $\phi$. As for the other problems studied in this paper, it turns out that the above problem is complete for the second level of the polynomial hierarchy, more specifically for $\Sigma_2^p$~\cite{Liberatore05}. Treewidth restrictions seem to not have been considered before, but we show that our approach gives a linear time algorithm in a simple way.

\begin{theorem}\label{thm:MUSupper}
 There is an algorithm that, given a CNF-formula $\phi$ incidence treewidth $k$ and a clause $C$ of $\phi$, decides $C$ is in a MUS of $\phi$ in time $2^{2^{O(k)}}|\phi|$.
\end{theorem}
\begin{proof}
Note that $C$ is in a MUS of $\phi$ if and only if there is an unsatisfiable clause set $\mathcal{C}$ such that $C\in \mathcal{C}$ and $\mathcal C\setminus \{C\}$ is satisfiable. We will encode this in $\forall\exists$-QBF. In a first step, similarly to the proof of Lemma~\ref{lem:combination}, we add a new variable $x_C$ for every clause $C$ of $\phi$ and substitute $\phi$ by clauses expressing $C\leftrightarrow x_c$. Call the resulting formula $\psi$. It is easy to see that the incidence treewidth of $\psi$ is at most double that of $\phi$. Moreover, for every assignment $a$ to $\var(\phi)$, there is exactly one extension to a satisfying assignment $a'$ of $\psi$. Moreover, in $a'$ a clause variable $x_C$ is true if and only if $a$ satisfies the clause $C$. Let $\mathcal{C}$ be a set of clauses, then $\mathcal C$ is unsatisfiable if and only if for every assignment $a$ to $\var(\phi)$, $\mathcal C$ is not contained in the set of satisfied clauses. Interpreting sets by assignments as before, we can write this as a formula by 
\[
 \psi'(\mathcal C):= \forall X \forall \mathcal C': \psi_C(X, \mathcal C') \rightarrow \neg (\mathcal C \subseteq \mathcal C').
\]
Let now $\mathcal C$ range over the sets of clauses not containing $C$. Then we have by the considerations above that $C$ appears in a MUS if and only if
\begin{align*}
 \psi^* &= \exists \mathcal C \psi'(\mathcal C \cup \{C\}) \land \neg \psi'(\mathcal C)\\
 &= \exists \mathcal C \exists X' \exists \mathcal C' \forall X'' \forall \mathcal C'' (\phi_C(X', \mathcal C') \rightarrow \neg (\mathcal C\cup \{C\} \subseteq \mathcal C')) \land \phi_C(X', \mathcal C'') \land \mathcal C\subseteq \mathcal C'' 
\end{align*}
Negating and rewriting the matrix of the resulting QBF with Lemma~\ref{lem:combination}, we get in linear time a $\forall \exists$-QBF of treewidth $O(k)$ that is true if and only if $C$ does not appear in a MUS of $\phi$.
Using Theorem~\ref{cor:hubie} completes the proof.
\end{proof}

We remark that different QBF encodings for MUS mumbership have also been studied in~\cite{JanotaS11}.
We now show that Theorem~\ref{thm:MUSupper} is essentially tight.

\begin{theorem}
   There is no algorithm that, given a CNF-formula $\phi$ with $n$ variables and primal treewidth $k$ and a clause $C$ of $\phi$, decides if $C$ is in a MUS of $\phi$ in time $2^{2^{o(k)}} 2^{o(n)}$, unless ETH is false.
\end{theorem}
\begin{proof}
 Given a $\forall\exists$-QBF $\psi = \forall X\exists Y \phi$ of incidence treewidth $k$ where $C_1, \ldots, C_m$ are the clauses of $\phi$, we construct the CNF-formula
 \[ \phi' = \bigwedge_{x\in X} (x \land \neg x) \land w \land \bigwedge_{i=1}^m (\neg w \lor C_i).\]
 In~\cite{Liberatore05} it is shown that $\psi$ is true if and only if the clause $w$ appears in a MUS of $\phi'$. Note that $\phi'$ has primal treewidth $k+1$: in a tree decomposition of the primal graph of $\phi$, we can simply add the variable $w$ into all bags to get a tree decomposition of the primal graph of $\phi'$. Since clearly $|\phi'| = O(|\phi|)$, any algorithm to check if $w$ is in a MUS of $\phi'$ in time $2^{2^{o(k)}} 2^{o(n)}$ contradicts ETH with Theorem~\ref{thm:lampisM}.
\end{proof}

\section{Conclusion}

In this paper, we took an alternate approach in the design of optimal algorithms mainly for the second level of the polynomial hierarchy parameterized by treewidth: we used reductions to 2-QBF.% as an algorithmic tool to construct algorithms for problems parameterized by treewidth. 
We stress that, apart from some technical transformations on CNF-formulas which we reused throughout the paper, our algorithms are straightforward and all complexity proofs very simple. We consider this as a strength of what we propose and not as a lack of depth, since our initial goal was to provide a black-box technique for designing optimal linear-time algorithms with an asymptotically optimal guarantee on the treewidth. We further supplement the vast majority of our algorithms by tight lower-bounds, using ETH reductions again from 2-QBF. %After all, the goal was to find algorithms in a simple way that still have essentially optimal runtime bounds. We thus consider the simplicity of the constructions as a great success.

We concentrated on areas of artificial intelligence, investigating a collection of well-studied and diverse problems that are complete for $\Sigma_2^p$ and $\Pi_2^p$. However we conjecture that we could apply our approach to several problems with similar complexity status. Natural candidates are problems complete for classes in the polynomial hierarchy, starting from the second level, see e.g.~\cite{SchaeferU02} for an overview (mere $\mathsf{NP}$-complete problems can often be tackled by other successful techiques). %simply because there are many $\Sigma_2^p$- and $\Pi_2^p$-complete problems there to consider. Inside artificial intelligence we focused on problems that we perceived as important and diverse, but ultimately our choice is of course biased by our background and our environment  and thus somewhat arbitrary. As a consequence, we encourage the readers to try other problems that they encounter and that are maybe closer to their heart.

Of course, our approach is no silver bullet that magically makes all other techniques obsolete. On the one hand, for problems whose formulation is more complex than what we consider here, Courcelle's Theorem might offer a richer language to model problems than QBF. This is similar in spirit to some problems being easier to model in declarative languages like ASP than in CNF. On the other hand, handwritten algorithms probably offer better constants than what we get by our approach. For example, the constants in~\cite{DvorakPW12} are more concrete and smaller than what we give in Section~\ref{sec:argumentation}. However, one could argue that for double-exponential dependencies, the exact constants probably do not matter too much simply because already for small parameter values the algorithms become infeasible\footnote{To give the reader an impression:  $2^{2^5} \approx 4.2\times 10^9$ and already $2^{2^6} \approx 1.8\times 10^{19}$.}. Despite these issues, in our opinion, QBF encodings offer a great trade-off between expressivity and tightness for the runtime bounds and we consider it as a valuable alternative.

\subsection*{Acknowledgments}
Most of the research in this paper was performed during a stay of the first and third authors at CRIL that was financed by the project PEPS INS2I 2017 CODA. The second author is thankful for many valuable discussions with members of CRIL, in particular Jean-Marie Lagniez, Emmanuel Lonca and Pierre Marquis, on the topic of this article.

Moreover, the authors would like to thank the anonymous reviewers whose numerous helpful remarks allowed to improve the presentation of the paper.
\bibliographystyle{plain}
\bibliography{twreasoning}

\begin{thebibliography}{10}

\bibitem{ArieliC12}
Ofer Arieli and Martin W.~A. Caminada.
\newblock A general qbf-based formalization of abstract argumentation theory.
\newblock In Bart Verheij, Stefan Szeider, and Stefan Woltran, editors, {\em
  Computational Models of Argument, {COMMA} 2012}, pages 105--116, 2012.

\bibitem{ArnborgP89}
Stefan Arnborg and Andrzej Proskurowski.
\newblock Linear time algorithms for {NP}-hard problems restricted to partial
  k-trees.
\newblock {\em Discrete Applied Mathematics}, 23(1):11--24, 1989.

\bibitem{AtseriasO14}
Albert Atserias and Sergi Oliva.
\newblock Bounded-width {QBF} is {PSPACE}-complete.
\newblock {\em J. Comput. Syst. Sci.}, 80(7):1415--1429, 2014.

\bibitem{Bodlaender96}
Hans~L. Bodlaender.
\newblock A linear-time algorithm for finding tree-decompositions of small
  treewidth.
\newblock {\em {SIAM} J. Comput.}, 25(6):1305--1317, 1996.

\bibitem{CadoliL94}
Marco Cadoli and Maurizio Lenzerini.
\newblock The complexity of propositional closed world reasoning and
  circumscription.
\newblock {\em J. Comput. Syst. Sci.}, 48(2):255--310, 1994.

\bibitem{Chen04}
Hubie Chen.
\newblock Quantified constraint satisfaction and bounded treewidth.
\newblock In Ramon~L{\'{o}}pez de~M{\'{a}}ntaras and Lorenza Saitta, editors,
  {\em Proceedings of the 16th Eureopean Conference on Artificial Intelligence,
  ECAI 2004}, pages 161--165, 2004.

\bibitem{Courcelle90}
Bruno Courcelle.
\newblock The monadic second-order logic of graphs. i. recognizable sets of
  finite graphs.
\newblock {\em Inf. Comput.}, 85(1):12--75, 1990.

\bibitem{challenge1993satisfiability}
DIMACS.
\newblock {Satisfiability: Suggested Format}.
\newblock {\em DIMACS Challenge. DIMACS}, 1993.

\bibitem{Dung95}
Phan~Minh Dung.
\newblock On the acceptability of arguments and its fundamental role in
  nonmonotonic reasoning, logic programming and n-person games.
\newblock {\em Artif. Intell.}, 77(2):321--358, 1995.

\bibitem{Dunne07}
Paul~E. Dunne.
\newblock Computational properties of argument systems satisfying
  graph-theoretic constraints.
\newblock {\em Artif. Intell.}, 171(10-15):701--729, 2007.

\bibitem{DunneB02}
Paul~E. Dunne and Trevor J.~M. Bench{-}Capon.
\newblock Coherence in finite argument systems.
\newblock {\em Artif. Intell.}, 141(1/2):187--203, 2002.

\bibitem{DvorakPW12}
Wolfgang Dvor{\'{a}}k, Reinhard Pichler, and Stefan Woltran.
\newblock Towards fixed-parameter tractable algorithms for abstract
  argumentation.
\newblock {\em Artif. Intell.}, 186:1--37, 2012.

\bibitem{EglyW06}
Uwe Egly and Stefan Woltran.
\newblock Reasoning in argumentation frameworks using quantified boolean
  formulas.
\newblock In Paul~E. Dunne and Trevor J.~M. Bench{-}Capon, editors, {\em
  Computational Models of Argument, {COMMA} 2006}, pages 133--144, 2006.

\bibitem{EibenGO16}
Eduard Eiben, Robert Ganian, and Sebastian Ordyniak.
\newblock Using decomposition-parameters for {QBF:} mind the prefix!
\newblock In Dale Schuurmans and Michael~P. Wellman, editors, {\em Proceedings
  of the Thirtieth {AAAI} Conference on Artificial Intelligence}, pages
  964--970, 2016.

\bibitem{abs-1711-02120}
Eduard Eiben, Robert Ganian, and Sebastian Ordyniak.
\newblock Small resolution proofs for {QBF} using dependency treewidth.
\newblock {\em CoRR}, abs/1711.02120, 2017.

\bibitem{EiterG93}
Thomas Eiter and Georg Gottlob.
\newblock Propositional circumscription and extended closed-world reasoning are
  {$\Pi_2^p$}-complete.
\newblock {\em Theor. Comput. Sci.}, 114(2):231--245, 1993.

\bibitem{EiterG95}
Thomas Eiter and Georg Gottlob.
\newblock The complexity of logic-based abduction.
\newblock {\em J. {ACM}}, 42(1):3--42, 1995.

\bibitem{FischerMR08}
Eldar Fischer, Johann~A. Makowsky, and Elena~V. Ravve.
\newblock Counting truth assignments of formulas of bounded tree-width or
  clique-width.
\newblock {\em Discrete Applied Mathematics}, 156(4):511--529, 2008.

\bibitem{FrickG04}
Markus Frick and Martin Grohe.
\newblock The complexity of first-order and monadic second-order logic
  revisited.
\newblock {\em Ann. Pure Appl. Logic}, 130(1-3):3--31, 2004.

\bibitem{GottlobPW10}
Georg Gottlob, Reinhard Pichler, and Fang Wei.
\newblock Bounded treewidth as a key to tractability of knowledge
  representation and reasoning.
\newblock {\em Artif. Intell.}, 174(1):105--132, 2010.

\bibitem{IgnatievPLM15}
Alexey Ignatiev, Alessandro Previti, Mark~H. Liffiton, and Joao
  Marques{-}Silva.
\newblock Smallest {MUS} extraction with minimal hitting set dualization.
\newblock In {\em Principles and Practice of Constraint Programming, {CP}
  2015}, volume 9255, pages 173--182, 2015.

\bibitem{JaklPRW08}
Michael Jakl, Reinhard Pichler, Stefan R{\"{u}}mmele, and Stefan Woltran.
\newblock Fast counting with bounded treewidth.
\newblock In {\em Logic for Programming, Artificial Intelligence, and
  Reasoning, {LPAR} 2008}, volume 5330, pages 436--450, 2008.

\bibitem{JanotaS11}
Mikol{\'{a}}s Janota and Jo{\~{a}}o P.~Marques Silva.
\newblock On deciding {MUS} membership with {QBF}.
\newblock In {\em Principles and Practice of Constraint Programming, {CP}
  2011}, volume 6876 of {\em Lecture Notes in Computer Science}, pages
  414--428. Springer, 2011.

\bibitem{Kloks94}
Ton Kloks.
\newblock {\em Treewidth, Computations and Approximations}, volume 842 of {\em
  Lecture Notes in Computer Science}.
\newblock Springer, 1994.

\bibitem{LagniezB13}
Jean{-}Marie Lagniez and Armin Biere.
\newblock Factoring out assumptions to speed up {MUS} extraction.
\newblock In Matti J{\"{a}}rvisalo and Allen~Van Gelder, editors, {\em Theory
  and Applications of Satisfiability Testing, {SAT} 2013}, volume 7962 of {\em
  Lecture Notes in Computer Science}, pages 276--292. Springer, 2013.

\bibitem{LagniezLM15}
Jean{-}Marie Lagniez, Emmanuel Lonca, and Jean{-}Guy Mailly.
\newblock Coquiaas: {A} constraint-based quick abstract argumentation solver.
\newblock In {\em 27th {IEEE} International Conference on Tools with Artificial
  Intelligence, {ICTAI} 2015}, pages 928--935, 2015.

\bibitem{LampisM17}
Michael Lampis and Valia Mitsou.
\newblock Treewidth with a quantifier alternation revisited.
\newblock 2017.

\bibitem{LangerRRS14}
Alexander Langer, Felix Reidl, Peter Rossmanith, and Somnath Sikdar.
\newblock Practical algorithms for {MSO} model-checking on tree-decomposable
  graphs.
\newblock {\em Computer Science Review}, 13-14:39--74, 2014.

\bibitem{Liberatore05}
Paolo Liberatore.
\newblock Redundancy in logic {I:} {CNF} propositional formulae.
\newblock {\em Artif. Intell.}, 163(2):203--232, 2005.

\bibitem{Lifschitz:1994:CIR:186124.186130}
Vladimir Lifschitz.
\newblock Handbook of logic in artificial intelligence and logic programming
  (vol. 3).
\newblock chapter Circumscription, pages 297--352. Oxford University Press,
  Inc., New York, NY, USA, 1994.

\bibitem{LokshtanovMS11}
Daniel Lokshtanov, D{\'{a}}niel Marx, and Saket Saurabh.
\newblock Lower bounds based on the exponential time hypothesis.
\newblock {\em Bulletin of the {EATCS}}, 105:41--72, 2011.

\bibitem{MarxM16}
D{\'{a}}niel Marx and Valia Mitsou.
\newblock Double-exponential and triple-exponential bounds for choosability
  problems parameterized by treewidth.
\newblock In {\em 43rd International Colloquium on Automata, Languages, and
  Programming, {ICALP} 2016}, pages 28:1--28:15, 2016.

\bibitem{McCarthy86}
John McCarthy.
\newblock Applications of circumscription to formalizing common-sense
  knowledge.
\newblock {\em Artif. Intell.}, 28(1):89--116, 1986.

\bibitem{SamerS10}
Marko Samer and Stefan Szeider.
\newblock Algorithms for propositional model counting.
\newblock {\em J. Discrete Algorithms}, 8(1):50--64, 2010.

\bibitem{SchaeferU02}
Marcus Schaefer and Christopher Umans.
\newblock Completeness in the polynomial-time hierarchy: A compendium.
\newblock {\em SIGACT news}, 33(3):32--49, 2002.

\bibitem{SilvaL11}
Jo{\~{a}}o P.~Marques Silva and In{\^{e}}s Lynce.
\newblock On improving {MUS} extraction algorithms.
\newblock In {\em Theory and Applications of Satisfiability Testing, {SAT}
  2011}, pages 159--173, 2011.

\end{thebibliography}

\begin{appendix}

\section{Proof of Proposition \ref{prop:threeCNF}}\label{app:threeCNF}

In this section we prove that given a CNF formula with bounded incidence treewidth and unbounded arity we can compute in linear time a 3CNF formula which has essentially the same treewidth (give or take a constant). 

\begin{proposition}
There is an algorithm that, given a CNF formula $\phi$ of incidence treewidth $k$, computes in time $2^{O(k)} |\phi|$ a 3CNF formula $\phi'$ of incidence treewidth $O(k)$ with $\var(\phi) \subseteq \var(\phi')$ such that $\phi$ is a projection of $\phi'$.
\end{proposition}

\begin{proof}
We use the classic reduction from SAT to 3SAT that cuts big clauses into smaller clauses by introducing new variables. During this reduction we have to take care that the runtime is in fact linear and that we can bound the treewidth appropriately. 

In a first step, we compute a tree decomposition $(T, (B_t)_{t\in T})$ of width $k$ for the incidence graph of $\phi$ in time $2^{O(k)}|\phi|$ with the algorithm from~\cite{BodlaenderDDFLP13}. %Since we will introduce additional factors later on anyway that we hide in $O(k)$-terms, let us for simplicity assume that $k$ is the width of the decomposition. 

We store $\phi$ in the following format: for every clause $C$ we have a doubly linked list $L_C$ storing pointers to all variables in $C$ and their polarity, i.e., if they appear positively or negatively. Moreover, for every variable $x$ we have a doubly linked list $L_x$ storing pointers to the clauses that $x$ appears in. So far, this is essentially an adjacency list representation of the incidence graph of $\phi$. As additional information, we add for every entry in $L_C$ that points to a variable $x$ also a pointer to the cell in $L_x$ that points to $C$. Symmetrically, we add in the cell pointing to $C$ in $L_x$ a pointer to the cell in $L_c$ pointing towards $x$. The purpose of this data structure is that it allows us efficient deletions: if we are in a cell in $L_x$ that points towards $C$, we can delete the edge $xC$ from the adjacency list in constant time without having to search for the right entry in $L_C$. Similarly, if we have the entry in $L_C$ representing the edge $xC$, we can delete this edge in constant time. Note that this data structure can be computed easily in linear time in a single pass over $\phi$.

We now construct the formula $\phi'$ along a postfix DFS order on $T$ (that can easily be computed in linear time). For each clause node $C$ appearing in a bag, we store a variable $F_{B,C}$ that will be put in the next clause we print out for $C$. $F_{B,C}$ will always be a variable that does not appear outside the subtree below $B$ in $T$ and $F_{B,C}$ might be empty. 

We now describe the construction in the different types of nodes $B$ in $T$:
\begin{itemize}
\item if $B$ introduces a new variable, we copy the values $F_{B,C}$ from its child and do nothing else.
\item if $B$ introduces a new clause $C$, we initialize $F_{B,C}$ as empty and copy all other values $F_{B,C'}$ from the child node.
\item if $B$ is forget node for a variable $x$, we first copy all $F_{B,C}$ as before. Then, for every clause $C$ in $B$ such that $C$ contains $x$, we do the following: if $F_{B,C}$ is empty, we set $F_{B,C}$ to $x$ with the polarity as in $C$. If $F_{B,C}$ contains a literal $\ell$, we write out a clause $\ell \lor \ell_x \lor z$ where $\ell_x$ is the variable $x$ with the same polarity as in $C$ and $z$ is a fresh variable we have not used before. Then we set $F_{B,C}$ to $\neg z$. Finally, in any case, we delete the edge $xC$ in our data structure. 
\item if $B$ is a forget node for a clause $C$, we again first copy all $F_{B,C}$. Then define the set $S_C$ that consists of $F_{B,C}$ and all literals whose variables are in $B$ that appear in $C$ . We arbitrarily split $S$ into a 3CNF by adding some more fresh variables and print it out. Afterwards, we delete all edges containing $x$ in our data structure.
\item if $B$ is a join node with two children $B_1$ and $B_2$, we compute the $F_{B,C}$ as follows: if $F_{B_1, C}$ and $F_{B_2, C}$ are empty, we set $F_{B,C}$ empty as well. If exactly one of the $F_{B_1, C}$ and $F_{B_2, C}$ contains a literal, we set $F_{B,C}$ to that literal. If $F_{B_1, C}$ contains $\ell_1$ and $F_{B_2, C}$ contains $\ell_2$, we print out a clause $\ell_1 \lor \ell_2 \lor z$ for a fresh variable $z$ and set $F_{B, C}$ to $\neg z$.
\end{itemize}
This completes the algorithm. The clauses we have printed out in the various steps form the formula $\phi'$. We have to check that the algorithm indeed runs in linear time in $|\phi|$ and that $\phi'$ has the desired properties.

Obviously, $\phi'$ is in 3CNF, because all clauses we print out only contain at most $3$ variables.

We next argue that the treewidth of $\phi'$ is $O(k)$. To this end, we construct a tree decomposition $(T, (B_t')_{t\in T}$. For every bag $B_t$ the corresponding bag $B_t'$ contains all variables in $B$ and all variables in the $F_{B_t,C}$. Moreover, for every $B_t$ for which the algorithm prints out a clause $C'$, we put $C'$ and the variables of $C'$ into $B_t'$. Since in every bag we print out at most $k$ clauses and all of them have size at most $3$, the resulting $B_t'$ has size $O(k)$. By construction, the bags $B_t'$ cover all edges in the incidence graph of $\phi'$. Finally, the connectivity condition is easy to verify. It follows that the treewidth of $\phi'$ is $O(k)$.

We next claim that the construction of $\phi'$ from $\phi$ and $(T, (B_t)_{t\in T})$ can be done in time $O(\poly(k) |\phi|)$ with the help of our data structure. To see this, first observe when a variable is forgotten in a bag $B$, it contains only at most $k$ neighbors in our data structure and those all lie in $B$. This is because for all neighbors that have been forgotten before, the corresponding edges have been deleted in the adjacency lists. The same is true when forgetting clauses. Thus, we can find the clauses that a variable is in in time $O(k)$. Since we can delete edges in constant time, it is easy to see that every bag can be treated in time polynomial in $k$. Since $T$ can be assumed to be of size linear in $|\phi|$, the desired runtime bound follows.

Finally, it is easy to see that $\phi$ is a projection of $\phi'$. This follows exactly as in the usual reduction from SAT to 3SAT. The only slight difference is that instead of cutting of pieces of the formula from the left to the right side, we decompose clauses potentially in a treelike fashion which results in clauses that contain only fresh variables. However, this does changes neither the correctness of the reduction not the argument. 
\end{proof}

\section{Bounded Treewidth $k$-QBF in Linear Time}\label{app:hubie}

We sketch a proof for linear runtime in Theorem~\ref{thm:hubie} and refer the reader to \cite{Chen04} for the technical details. In fact, since this will not be much more work, we treat the case of $r$-QBF, the generalization of $2$-QBF to $r$ quantifier blocks. 

To give the runtime bound, define $g(r, k)$ recursively by $g(0,k) = k$ and $g(r+1, k) = 2^{g(k,r)}$. We now state the linear time version of the main result in~\cite{Chen04}.

\begin{theorem}\label{thm:hubiegeneral}
 There is an algorithm that given a $r$-QBF of primal treewidth $k$ decides in time $g(r,O(k))|\phi|$ if $\phi$ is true.
\end{theorem}
We remark that Theorem~\ref{thm:hubiegeneral} was already observed in~\cite{PanV06} but, since that paper gives no justification of that claim, we decided to give some more details.

The crucial data structure in~\cite{Chen04} are \emph{choice constraints} which consist of a variable scope and a rooted tree with unbounded fanout in which all leaves are at depth $r$ and for every leaf a relation $R\subseteq \{0,1\}^s$. 

Chen then defines \emph{choice quantified formulas} which consist of a variable prefix and a conjunction of choice constraints in the variables of the prefix. We omit the semantics of choice quantified formulas here since they are not important for our sketch and refer the reader to~\cite{Chen04} for details. We remark however that any given QBF $\phi$ where all clauses have at most $k$ variables can be turned in time $O(2^k n)$ into a choice quantified formula $\psi$ such that $\phi$ is true if and only if $\psi$ is.

We define a notion of equivalence for nodes in choice constraints: leaves are equivalent if and only if they have the same relation. Equivalence of two nodes $t, t'$ of depth $\ell<r$ is defined recursively: let $t_1\ldots, t_s$ be the children of $t$ and $t_1', \ldots, t_s'$ be children of $t'$. Then $t$ and $t'$ are equivalent if and only if for every $t_i$ there is an equivalent $t_j'$ and for every $t_j'$ there is an equivalent $t_i$\footnote{We remark that the notion of equivalence in~\cite{Chen04} is slightly different. We chose to modify the definition since our notion gives slightly smaller size bounds for normalized choice constraints and have the same properties concerining truth of the resulting formulas.}. A choice constraint is in normal form if and only if no node has any equivalent children.

Chen shows that one can \emph{normalize} a choice constraint by deleting iteratively for all equivalent pairs of nodes one of them and its subtree. Crucially, applying this operation for a constraint in a choice quantified formula yields an equivalent formula. 

\begin{observation}\label{obs:hubiesize}
There are $g(r+1,k)$ non-equivalent normal choice constraints with scope of size $k$ whose leaves are at depth $r$ and all of these constrains have size $g(r, 2k)2^k$.
\end{observation}
\begin{proof}
Choice constraints of depth $0$ are just relations of arity $k$ over $\{0,1\}$, so there are $g(1,k) = 2^k$ of those. Moreover, each of those relations can be decribed in size $2^{O(k)}$, e.g. by a value table.

For $r>0$, by definition of equivalence, the root can have as children any subset of normal choice constraints of depth $r-1$ and the same scope. Since there are by induction $g(r,k)$ of those, the first claim follows by definition of $g$. The description size is at most $g(r,k) \cdot g(r-1,2k) 2^k \le g(r,2k)2^k$
\end{proof}

Note that we can check the equivalence of two children in time polynomial in the size of a given choice constraint and thus normalization can also be done in time $\poly(g(r, 2k)2^k) = g(r,O(k))$.

Chen also introduces a polynomial time computable join operation on choice constraints with the property that substituting two choice constraints by their join in a choice quantified formula one gets a new formula that is equivalent to the old. A naive solution to solve choice quantified formulas would thus be to simply join all its choice constraints and then check the single resulting constraint. The problem with this is that it would grow the variables scope of the resulting constraint (the variables of a join are the unions those of the joined constraints) such that the size bound in Observation~\ref{obs:hubiesize} would become meaningless and the runtime bound would explode.

The solution to this is working along a tree decomposition: in every forget node, one joins all choice constraints having the forgotten variable in their scope. Note that these choice constraints and thus also the resulting join only have at most the $k+1$ variables of the current bag in the scope, so by normalizing after every join, the resulting choice constraint will have size at most $g(r, O(k))$. Now the variable to forget only appears in a single choice constraint and Chen shows how to forget it in that case in an operation that may grow the choice constraint but by applying normalization during this forget operation one maintains the size bound of Observation~\ref{obs:hubiesize}. Applying this on all forget nodes iteratively, one gets a trivial choice constraint formula that can be decided in constant time.

Let us explain why this algorithm runs within the claimed time bounds: by similar preprocessing than that in Proposition~\ref{prop:threeCNF}, we can make sure that for every node that is forgotten, we can look up all constraints it appears in time linear in the number of those constraints, so at most $O(2^{k})$. Now we compute at most $2^k$ pairwise joins followed by normalization which run each in time $g(r, O(k))$. Thus the overall time for joining and normalizing is $2^kg(r,O(k))= g(r,O(k))$. Observing that the forgetting can also be done in polynomial time and thus in $g(r, O(k))$ leads to an overall cost per forgotten variables of $g(r, O(k))$. Now noting that the computation of the tree decomposition and a traversal to find the order in which the variables are forgotten can be done in time $2^{O(k)} n$ where $n$ is the number of variables, completes the proof.
% \section{Valia's version}
% \input{hubie.tex}

\end{appendix}

\end{document}